\pdfoutput=1
\documentclass[11pt]{article}

\usepackage{emnlp2021}

\usepackage{times}
\usepackage{latexsym}
\usepackage{enumerate}
\usepackage{enumitem}
\usepackage{color}
\usepackage{adjustbox}
\usepackage{booktabs}
\usepackage{makecell}
\usepackage{amsfonts}
\usepackage{amsthm}
\usepackage{amsmath}
\usepackage{amssymb}
\usepackage{multirow}
\usepackage{xfrac}
\usepackage{scalerel}
\usepackage{float}
\usepackage{inconsolata} 
\usepackage[belowskip=-10pt,aboveskip=10pt]{caption}

\usepackage[disable]{todonotes}
\makeatletter
\newcommand*\iftodonotes{\if@todonotes@disabled\expandafter\@secondoftwo\else\expandafter\@firstoftwo\fi}  % defines \iftodonotes{<true>}{<false>}, thanks to https://tex.stackexchange.com/questions/126559/conditional-based-on-packageoption
\makeatother

\newcommand{\note}[4][]{\todo[author=#2,color=#3,size=\scriptsize,fancyline,caption={},#1]{#4}} % default note settings, used by macros below.
\newcommand{\ryan}[2][]{\note[#1]{ryan}{violet!40}{#2}}
\newcommand{\roger}[2][]{\note[#1]{roger}{green!40}{#2}}
\newcommand{\clara}[2][]{\note[#1]{clara}{orange}{#2}}
\newcommand{\tiago}[2][]{\note[#1]{tiago}{cyan!40}{#2}}

\renewcommand{\hat}{\widehat}

\usepackage{mathtools}

\newcommand{\citeposs}[1]{\citeauthor{#1}'s (\citeyear{#1})}

\DeclarePairedDelimiterX{\infdivx}[2]{(}{)}{%
  #1\;\delimsize\|\;#2%
}

\usepackage{microtype}
\usepackage{cleveref}
\crefname{section}{\S}{\S\S}
\Crefname{section}{\S}{\S\S}
\crefname{table}{Tab.}{}
\crefname{figure}{Fig.}{Figs.}
\crefname{algorithm}{Algorithm}{}
\crefname{equation}{Eq.}{Eqs.}
\crefname{line}{Line}{}
\crefname{appendix}{App.}{}
\crefformat{section}{\S#2#1#3}
\crefname{thm}{Theorem}{}
\crefname{cor}{Corollary}{}
\crefname{prop}{Proposition}{}
\crefname{def}{Definition}{}

\newtheorem{theorem}{Theorem}[section]

\newcommand{\ww}{\mathbf u}
\newcommand{\surp}{s}
\newcommand{\mulang}{\mu_{\mathrm{lang}}}
\newcommand{\musent}{\mu_{\mathrm{sent}}}
\newcommand{\dist}{\Delta}

\newcommand{\uidinv}{\textsc{uid}^{-1}}
\newcommand{\defeq}[0]{\mathrel{\stackrel{\textnormal{\tiny def}}{=}}}
\newcommand{\defn}[1]{\textbf{#1}}
\newcommand{\peffort}{\mathrm{Effort}}

\newcommand{\deltaLL}{\Delta\mathrm{LogLik}}
\newcommand{\accept}{\mathrm{Acceptability}}
\newcommand{\reading}{\mathrm{ReadingTime}}
\definecolor{darkgrey}{rgb}{0.2,0.2,0.2}
\definecolor{darkgreen}{rgb}{0.0,0.6,0.0}
\definecolor{darkblue}{rgb}{0.0,0.0,0.5}

\newcommand{\db}[1]{\textcolor{darkblue}{\bf\scriptstyle \selectfont \,(\pm#1)}}

\title{Revisiting the Uniform Information Density Hypothesis}

\usepackage{tipa}

\newcommand{\ucambridge}{2}
\newcommand{\uzh}{3}
\newcommand{\ethz}{1}
\newcommand{\up}{4}
\newcommand{\MIT}{5}

\author{Clara Meister$^{\ethz}$,~\;~Tiago Pimentel$^{\ucambridge}$,~\;~Patrick Haller$^{\uzh}$,~\;~ Lena Jäger$^{\uzh,\up}$, \\
\textbf{Ryan Cotterell$^{\ethz,\ucambridge}$,~\;~ Roger Levy$^{\MIT}$}\\
  $^{\ethz}$ETH Z\"{u}rich~\;~ $^{\ucambridge}$University of Cambridge~\;~$^{\uzh}$University of Zurich~\;~ \\
  $^{\up}$University of Potsdam~\;~$^{\MIT}$Massachusetts Institute of Technology \\
  
  \texttt{\href{mailto:clara.meister@inf.ethz.ch}{clara.meister@inf.ethz.ch}}~\;~ \texttt{\href{mailto:tp472@cam.ac.uk}{tp472@cam.ac.uk}}~\;~ \texttt{\href{mailto:haller@cl.uzh.ch}{haller@cl.uzh.ch}} \\
  \texttt{\href{mailto:jaeger@cl.uzh.ch}{jaeger@cl.uzh.ch}}~\;~ \texttt{\href{mailto:ryan.cotterell@inf.ethz.ch}{ryan.cotterell@inf.ethz.ch}}~\;~ \texttt{\href{mailto:rplevy@mit.edu}{rplevy@mit.edu}}
}

\date{}

\begin{document}
\maketitle
\begin{abstract}
The uniform information density (UID) hypothesis posits a preference among language users for utterances structured such that information is distributed uniformly across a signal. 
While its implications on language production have been well explored, the hypothesis potentially makes predictions about language comprehension and linguistic acceptability as well. Further, it is unclear how uniformity in a linguistic signal---or lack thereof---should be measured, and over which linguistic unit, e.g., the sentence or language level, this uniformity should hold.
Here we investigate these facets of the UID hypothesis using reading time and acceptability data. While our reading time results are generally consistent with previous work, they are also consistent with a weakly super-linear effect of surprisal, which would be compatible with UID's predictions. For acceptability judgments, we find clearer evidence that non-uniformity in information density is predictive of lower acceptability. We then explore multiple operationalizations of UID, motivated by different interpretations of the original hypothesis, and \roger{Presumably everything after this point in the abstract needs to be rewritten?}\clara{hmmm why is that? these results are from the last set of experiments}analyze the scope over which the pressure towards uniformity is exerted. 
The explanatory power of a subset of the proposed operationalizations suggests that the strongest trend may be a regression towards a mean surprisal across the language, rather than the phrase, sentence, or document---a finding that supports a typical interpretation of UID, namely that it is the byproduct of language users maximizing the use of a (hypothetical) communication channel.\footnote{Analysis pipeline is publicly available and can be found at \url{https://github.com/rycolab/revisiting-uid}.}\looseness=-1
\end{abstract}

\section{Introduction}
\begin{figure}[h!]
    \centering
    \includegraphics[width=\linewidth]{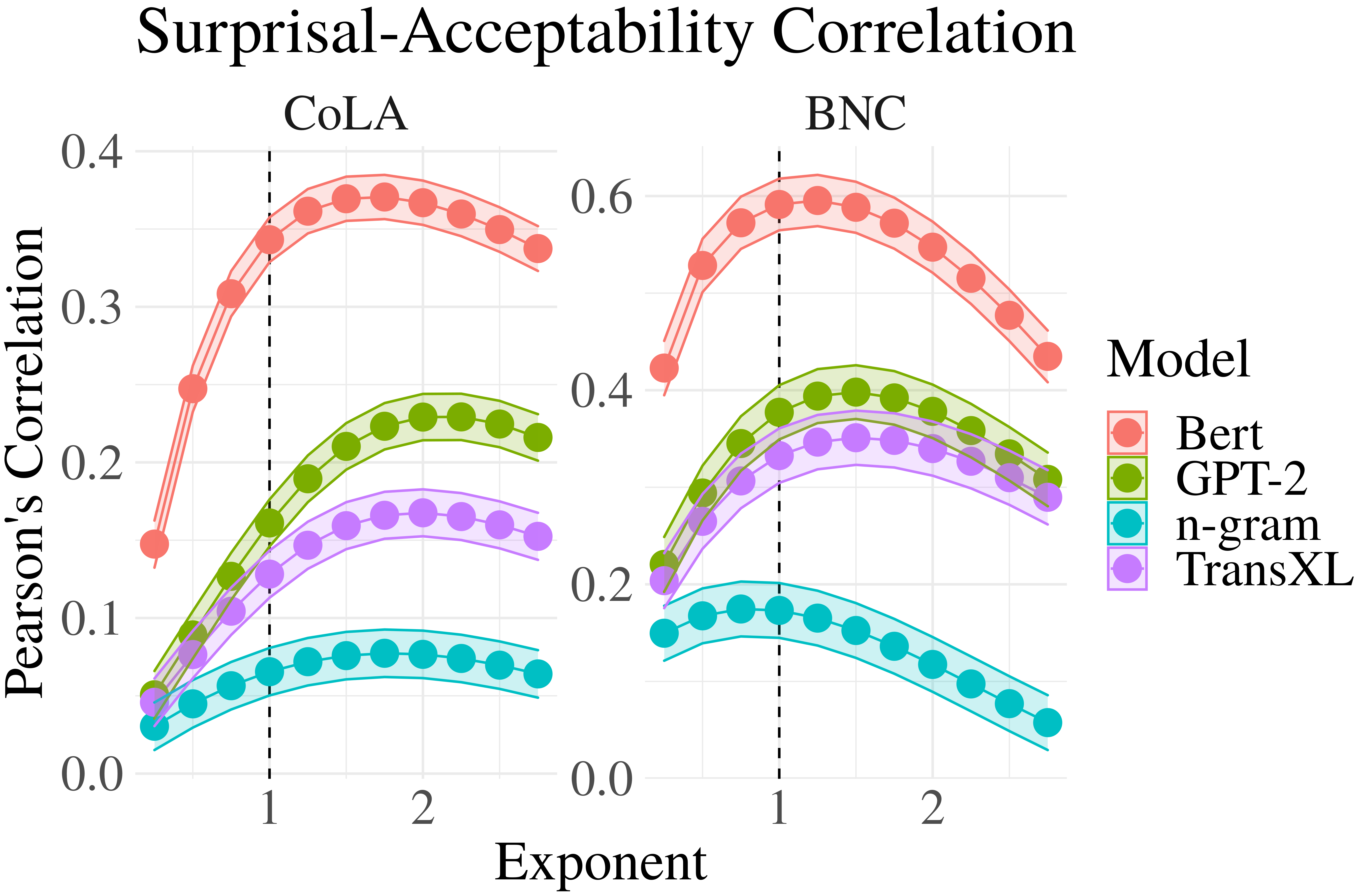}
    \caption{Correlation coefficient between (negative) sum of surprisals raised to the $k^{\text{th}}$ power and linguistic acceptability judgments of a sentence. 
    The higher correlation when $k>1$ implies sentences with a more uniform distribution of information are more acceptable.\looseness=-1
    } \label{fig:cola}
     \setlength{\belowcaptionskip}{-30pt}
\end{figure}

The uniform information density (UID) hypothesis \citep{fenk1980konstanz,levy2007speakers} states that language users prefer when information content (measured information-theoretically as \emph{surprisal}) is distributed as smoothly as possible throughout an utterance. The studies adduced in support of this hypothesis in language production span levels of linguistic structure: from phonetics \cite{aylett-2004} to lexical choice \cite{MAHOWALD2013313}, to syntax \cite{Jaeger2010RedundancyAR}, and to discourse (\citealt{asr-demberg:2015uniform-surprisal}) (though see \citealt{zhan-levy:2018naacl,zhan-levy:2019-availability}). Despite this evidence, there are several aspects of the UID hypothesis that lack clarity or unity. For example, there is a dearth of converging evidence from studies in language \emph{comprehension}. Furthermore, multiple candidate operationalizations of UID have been proposed, each without formal justification for their choices \cite{collins2014,jain-etal-2018-uniform,meister-etal-2020-beam,wei+al.acl2021}.\looseness=-1

In this work, we attempt to shed light on these issues: we first study the relationship between the distribution of information content throughout a sentence and native speakers' (i) sentence-level reading times and (ii) sentence acceptability judgments. While our results for sentence-level reading times do not contradict previous word-level reading time analyses (e.g., \citealt{smith2013-log-reading-time,goodkind-bicknell:2018predictive}), which have shown a linear effect of surprisal, they suggest that a slight super-linear effect may likewise be a plausible explanation---which is in line with predictions of the UID hypothesis. For sentence acceptability judgments, we see more concrete signs of a super-linear effect of sentence-level surprisal (see \cref{fig:cola}), consistent with a preference for UID in language.
Given these findings, we next ask how we can best measure UID. We review previous results supporting UID, in search of an operationalization and find that in most of these studies, adherence to UID is measured via an analysis of individual linguistic units, without direct consideration for the information content carried by \emph{surrounding} units \cite{Frank2008SpeakingRU,Jaeger2010RedundancyAR,MAHOWALD2013313}.
Such a definition fails to account for the distribution across the signal as a whole.

Consequently, we present and motivate a set of plausible operationalizations---either taken from the literature or newly proposed. Given our earlier results, we posit that good operationalizations of UID should provide strong explanatory power for human judgments of linguistic acceptability and potentially reading times. 
In this search, we additionally explore with respect to which linguistic unit---a phrase, sentence, document, or language as a whole---uniformity should be measured. 
Our results provide initial evidence that the best definition of UID may be a super-linear function of word surprisal. Further, we see that a regression towards the mean information content of the entire language, rather than a local information rate, may better capture the pressure for UID in natural language, a theory that falls in line with its information-theoretical interpretation, i.e., that language users maximize the use of a hypothetical noisy channel during communication.\tiago{Should we add a footnote here about our other paper in emnlp? Perhaps saying. "In x we investigate this same pressure cross-linguistically. Finding evidence for a regression towards a cross-linguistic mean information content."}\clara{If we have the space in the end}\looseness=-1

\section{Processing Effort in Comprehension}\label{sec:background}
In psycholinguistics, there are a number of theories that explain how the effort required to process language varies as a function of some perceived linguistic unit. Several of these are founded in information theory \cite{shannon1948mathematical}, using the notion of language as a communication system in order to build computational models of processing. 
Under such a framework, linguistic units convey information, and the exact amount of information a unit carries can be quantified as its \defn{surprisal}---also termed Shannon information content. Formally, let us consider a linguistic signal $\ww = \langle u_1, \dots, u_N\rangle$ as a sequence of linguistic units, e.g., words or morphemes; the standard definition of surprisal is then $\surp(u_n) \defeq - \log p(u_n \mid \mathbf u_{< n})$, i.e., a unit's negative log-probability conditioned on its prior context. Note that under this definition, low probability items are seen as more informative, which reflects the intuition that unpredictable items convey more information than predictable ones. With this background in mind, we now review two prominent examples of information-theoretic models of language processing: surprisal theory and the uniform information density hypothesis. 
\looseness=-1

\subsection{Surprisal Theory}
Surprisal theory \cite{hale-2001-probabilistic} posits that the incremental load of processing a word is directly related to how unexpected the word is in its context, i.e., its surprisal. Mathematically formulated, the processing effort required for the word $u_n$ follows a linear relationship with respect to its surprisal:\looseness=-1
\begin{equation}\label{processing_effort}
    \peffort(u_n)\propto s(u_n )
\end{equation}
Over the years, surprisal theory has been further motivated and received wide empirical support \cite{levy2008expectation,brouwer-etal-2010-modeling}.\footnote{
\citet{levy2008expectation} connects surprisal theory to resource reallocation---the effort required to update an internal probability distribution over possible parses during sentence comprehension.
\citet{brouwer-etal-2010-modeling} found that surprisal theory accounts for processing difficulty when disambiguating certain linguistic structures in Dutch.}
Notably, a number of works give evidence that this relationship (between processing effort and surprisal) is indeed linear (equivalently, logarithmic in probability; \citealt{smith2013-log-reading-time,frank-etal-2013-word,goodkind-bicknell-2018-predictive}, though see \citealt{BROTHERS2021104174}).

\subsection{Uniform Information Density}
Given the formal definition of surprisal, the information content of the entire linguistic signal $\ww$ can be quantified as the sum of individual surprisals. 
Following \cref{processing_effort}, the effort to process $\ww$ would thus be proportional to this sum, i.e.:
\begin{equation}\label{eq:cost_sum}
    \peffort(\ww)\propto \sum_{n=1}^N s(u_n)
\end{equation}
But this has a counter-intuitive consequence. Suppose a speaker has a fixed number of bits of information to convey. \cref{eq:cost_sum} predicts that \emph{all ways of distributing that information in an utterance would involve equal processing effort}: packing it all into a single, short utterance; spreading it out thinly in an extremely long utterance; dispersing it in a highly uneven profile throughout an utterance.\looseness=-1

The theory of uniform information density (UID; \citealt{fenk1980konstanz,genzel-charniak-2002-entropy,bell2003effects,aylett-2004,levy2007speakers}) attempts to reconcile the role of surprisal in determining processing effort with the intuition that perhaps not all ways of distributing information content have equal effect on overall processing effort. Rather, UID predicts that communicative efficiency is maximized when information---again quantified as per-unit surprisal---is distributed \emph{as uniformly as possible} throughout a signal. One way of deriving this prediction is to hypothesize that the processing effort for a sentence is an additive function of (i) a \emph{super-linear} function of surprisal; and (ii) utterance length:\footnote{See also Ch.2 of \citet{levy-thesis} and \citet{levy:2018cogsci} for more extensive discussion.}\looseness=-1
\begin{equation}\label{eq:cost}
    \peffort(\ww)\propto \sum_{n=1}^N s(u_{n})^k + c\cdot N 
\end{equation}
for some constant $c>0$ and $k>1$. 
The above equation implies that high surprisal instances require disproportionately high processing effort from the language user.  %\cite{chafe1987cognitive}. 
Rather, a uniform distribution of $\surp(u_n)$---which for fixed $N$ and total information is the unique minimizer of \cref{eq:cost}---would incur the least processing effort. Proof given in \cref{sec:unique-min-proof}.\looseness=-1

Due to its support by a number of studies, the UID hypothesis has received considerable recognition in the cognitive science community.
Such verifications, though, derive mostly from the tendencies implied by \cref{eq:cost}---as opposed to its direct verification. 
Take the original \citet{levy2007speakers} as an example: while they propose a formal operationalization of UID, they evaluate their hypothesis by analyzing a surprisal vs. sentence length trade-off rather than assessing the operationalization directly.
Furthermore, most UID studies investigate \emph{individual} word surprisals, without regard for their distribution within the sequence \citep[\textit{inter alia}]{aylett-2004,MAHOWALD2013313}.\looseness=-1

\section{Quantifying Linguistic Uniformity}\label{sec:sequence-level}

UID is, by its definition, a smoothing effect; it can be seen as a regression to a mean information rate---either measured as the surprisal per lexical unit (in written text, as we analyze here), or surprisal per time unit (in speech data).
% potentially measured as the surprisal per time.
However, there are multiple ways the hypothesis may be interpreted. 
As a concrete example, we turn to \citeposs{collins2014} fourth figure, which we recreate here in \cref{fig:collins}.
In its perhaps better-known form, UID suggests that language transmission should happen at a roughly constant rate, close to the channel capacity, i.e., there is a fixed (and perhaps cross-linguistic; \citealt{coupe2019different,pimentel+al.emnlp2021a}) value from which a unit's information density should never heavily deviate.
Under this interpretation, S1 (red) adheres more closely to UID, as information content per word varies less---in absolute terms---across the sentence. We can formalize this notion of UID using an inverse relationship to some per-unit distance metric $\dist(\cdot, \cdot)$ as follows:     \looseness=-1
\begin{equation}\label{eq:const_constant}
   \uidinv(\ww) = \frac{1}{N}\sum_{n=1}^{N} \dist(\surp(u_n), \mu_c)
\end{equation}
where $\mu_c$ is a target (mean) information rate---presumably at a theoretical channel's capacity. This mathematical relationship reflects the intuition that the further the units in a linguistic signal are from the average information rate $\mu_c$, the less the signal adheres to UID.

\begin{figure}
    \centering
    \includegraphics[width=\linewidth]{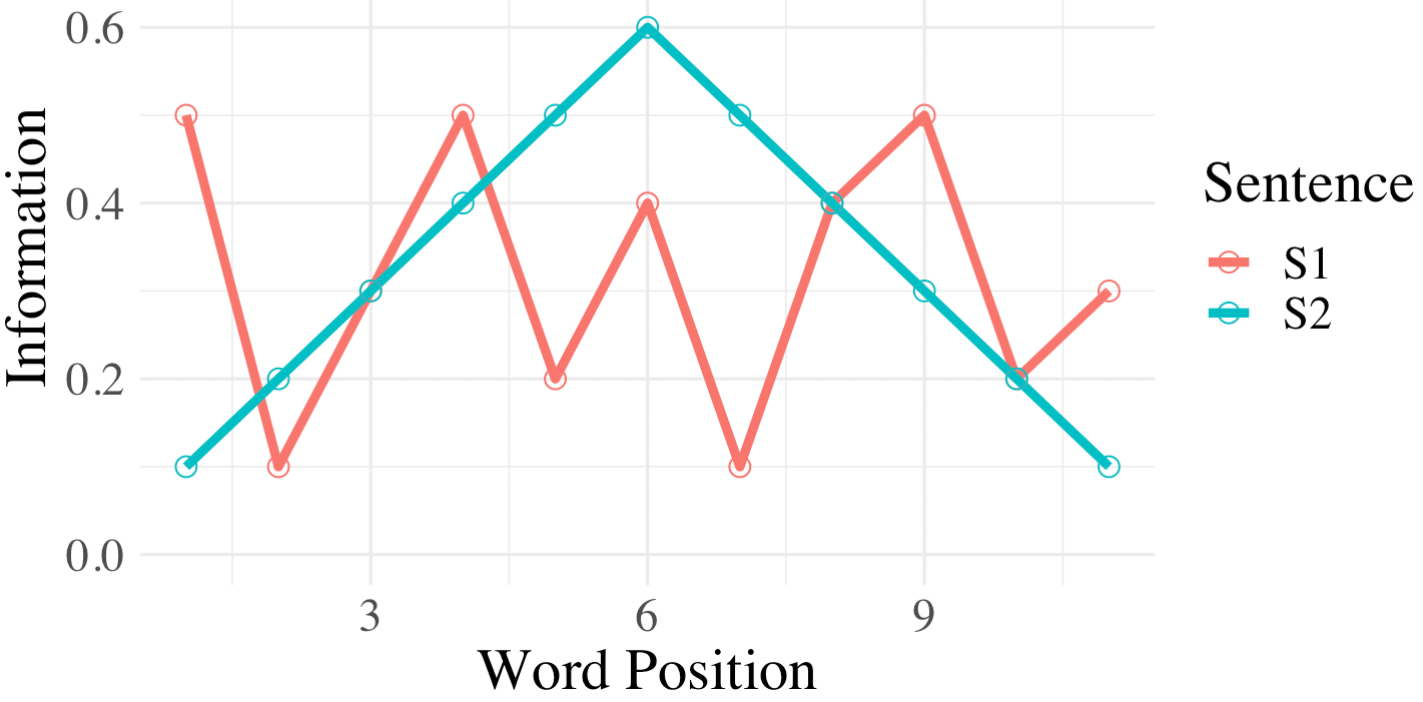}
    \caption{Information distribution across words of two hypothetical sentences. Recreation of Fig. 4 in \citeposs{collins2014}.} \label{fig:collins}
\end{figure}

% remains relatively low across the sentence. 
We may, however, also interpret UID as a pressure to avoid rapidly shifting from information dense (and therefore cognitively taxing) sections to sections requiring minimal processing effort. 
Rather, in an optimal setting, there should be a smooth transition between information sparse and dense components of a signal. 
Under this interpretation, we might believe S2 (blue) to adhere more closely to UID, as local changes are gradual. We can formalize this version of UID as
\begin{equation}\label{eq:const_variability}
   \uidinv(\ww) = \frac{1}{N-1}\sum_{n=2}^{N} \dist(\surp(u_n), \surp(u_{n-1}))
\end{equation}
The difference between these two is concisely summarized as minimizing  global vs. local variability. The former definition has arguably received more attention; studies such as \citet{Frank2008SpeakingRU}, among others, analyze UID through regression towards a \emph{global} mean. Yet, there are arguments that variability should instead be measured locally \cite{collins2014,bloem-2016-testing}.\looseness=-1

\subsection{Regressing to Which Mean?}

Notably, there is an aspect of the global variability presented in \cref{eq:const_constant} that remains underspecified: what exactly is $\mu_c$? A mean information rate may be with respect to a phrase, a sentence or even a language as a whole; this rate could even span across languages, a definition that nicely aligns with recent cross-linguistic experiments on spoken language data that argue for a universal channel capacity \cite{pellegrino2011,coupe2019different}. 
Yet, the former definitions likewise seem plausible.\looseness=-1

To motivate this argument, consider the relationship between \emph{cadence} in literary writing and UID. We loosely define cadence as the rhythm and speed of a piece of text, which should have a close relationship to the dispersion of information. When writing prose, authors typically vary cadence across sentences, interspersing short, impactful (i.e., high information) sentences within series of longer sentences to avoid repetitiveness. We have done so here, in this paper. Yet, intuitively, this practice does not lead to particularly high processing costs, at least for a native speaker. Indeed, some would argue that such fluctuations make text \emph{easier} to read. This example motivates a pull towards a more context-dependent---perhaps sentence-level---rather than language-level mean information rate.\looseness=-1

While a number of findings undoubtedly demonstrate a pressure against high (and sometimes even inordinately low) surprisal---which aligns with the first (global) interpretation of the UID hypothesis---their experimental setups, in general, do not provide evidence for or against a more local interpretation, such as the one just described.\footnote{We attribute this to the fact that most of these analyses were performed at the word- rather than sequence-level.} 
We now define a number of UID operationalizations that encompass these different interpretations, subsequently analyzing them in \cref{sec:exps}.

\subsection{Operationalizing UID}

The first operationalization on which we will focus follows from \cref{eq:cost}, suggesting a \defn{super-linear} effect of surprisal on processing effort:\looseness=-1
\begin{equation}\label{eq:power}
    \uidinv(\ww) = \frac{1}{N}\sum_{n=1}^N s(u_{t})^k \quad\quad {\color{gray} (k > 1)}
\end{equation}
where $k$ controls the strength of super-linearity. 
%Note that under, surprisal theory, we would expect $k=1$ to offer the best explanation of cognitive processing load. 

A second operationalization, similar to \cref{eq:const_constant}, implies a pressure for \defn{mean regression}:
\begin{equation}\label{eq:var}
   \uidinv(\ww) = \frac{1}{N}\sum_{n=1}^{N} (\surp(u_n) - \mu)^2
\end{equation}
Note that we may take $\mu$ from a number of different contexts. For example, $\musent = \frac{1}{N}\sum_{n=1}^{N} \surp(u_n)$ for sentence $\langle u_1, \dots, u_N \rangle$ implies a \emph{sentence-level} mean regression, 
whereas average surprisal over an entire language $\mulang$ suggests a regression to a (perhaps language-specific) channel capacity. Both definitions more closely align with our global interpretation of UID, i.e., that S1 (red) of \cref{fig:collins} may exhibit a more ``uniform'' distribution of information.\looseness=-1

Similarly, we can compute the \defn{local variance} in a sentence as\footnote{\cref{eq:lc,eq:max} were originally used in \citet{collins2014}.}
\begin{equation}\label{eq:lc}
  \uidinv(\ww) = \frac{1}{N-1}\sum_{n = 2}^{N} (\surp(u_{n}) - \surp(u_{n-1}))^2
\end{equation}
\noindent which, in contrast to \cref{eq:power}, aligns more with our local interpretation of UID. 

We may also interpret UID as a pressure to minimize a signal's \defn{maximum} per-unit surprisal, as this may be a point of inordinately high cognitive load for the comprehender:%
\begin{equation}\label{eq:max}
    \uidinv(\ww) = \overset{N}{\max_{n=1}}\,\surp(u_n)
\end{equation}
For completeness, we further propose another potential measure of UID compliance inspired by the information-theoretic nature of UID. We consider the \defn{R\'enyi entropy} \citep{renyi1961measures} of a probability distribution $p$, defined as:

\begin{equation}\label{eq:renyi_eq}
   \mathrm{H}_k(p) = \frac{1}{1-k}\log \sum_{x\in \mathcal{X}} p(x)^k
\end{equation}
where $\mathcal{X}$ is the support of the distribution $p$. 
Notably, the R\'enyi entropy, which is maximized when $p$ is uniform, becomes the Shannon entropy in the limit as $k \rightarrow 1$.\footnote{We adopt this definition $\mathrm{H}(p) = - \sum_{x} p(x)\log p(x)$ when referring to \cref{eq:renyi_eq} for $k=1$.} However, for $k > 1$, high probability items contribute disproportionately to this sum, which in our context, would translate to an emphasis on \emph{low}-surprisal items. Thus, we do not expect it to be a good operationalization of (inverse) UID.  
% Further, it is a monotonically non-increasing function of $k$; as such, for values of $k<1$ it will be larger or equal to its narrower counterpart (i.e. to Shannon's entropy)---with an identity only when information is uniformly distributed.
However, the opposite holds for $k<1$, where R\'enyi entropy can be seen as producing an extra cost for low-probability, i.e., high-surprisal items. 
Thus, in terms of UID, we take:
\begin{equation}\label{eq:renyi}
    \uidinv(\ww) = \begin{cases}
\mathrm{H}_k(\hat p) & \textbf{if}\,\,\,k < 1\\
\mathrm{H}_k^{-1}(\hat p)& \textbf{otherwise}
\end{cases}
\end{equation}
where $\hat p$ is a distribution over $u_1, \dots, u_N$ normalized to sum to 1.\footnote{Since $p(\cdot \mid \ww_{<t})$ for $\langle u_1, \dots, u_N\rangle$ is not in itself a probability distribution, we must renormalize in order for this metric to have the properties exhibited by entropy.} 

\subsection{UID, Effort and Acceptability}
We now revisit the processing effort of a sentence, rewriting it in terms of our UID operationalizations\looseness=-1
\begin{equation}\label{eq:peffort2}
    \peffort(\ww) \propto \uidinv(\ww)\cdot N + c\cdot N 
\end{equation}
i.e., processing effort is proportional to the interaction between (i.e., multiplication by) $\uidinv$ and sentence length. Note that when using our operationalization of UID from \cref{eq:power}, this equation reverts to \citeposs{levy-thesis} original \cref{eq:cost}. Further, this equation with $k=1$ and $c=0$ recovers the hypothesis under surprisal theory. Following previous work \citep[][\emph{inter alia}]{frank-bod,goodkind-bicknell:2018predictive}, we then model reading time as $\reading(\ww) \propto \peffort(\ww)$; in words, (proportionally) more time is taken to read more cognitively demanding sentences.

We further consider the relationship between UID and linguistic acceptability; we posit that
\begin{equation}\label{eq:accept}
    \accept^{-1}(\ww) \propto \uidinv(\ww)\cdot N 
\end{equation}
\noindent i.e., the linguistic acceptability of a sentence has an inverse relationship with processing effort (withholding the additional penalty for length). 
Intuitively, sentences that are easier to process are more probably acceptable sentences, and vice versa. While not comprehensive, there is evidence that this simple model (at least to some extent) captures the relationship between these two variables \cite{Topolinski2009TheAO}.
Given these models, we now evaluate our different operationalizations based on their predictive power of psychometric variables.\looseness=-1

\section{Experiments}
\paragraph{Data.}
We employ reading time data in English from 4 corpora over 2 modalities: the Natural Stories \cite{futrell-etal-2018-natural2} and Brown \cite{smith2013-log-reading-time} Corpora, which contain self-paced reading time data, as well as the Provo \cite{provo} and Dundee Corpora \cite{dundee}, which contain eye movements during reading.\footnote{We additionally perform experiments using the GECO dataset \cite{geco}, an eye-tracking corpus with Dutch data. These results are shown in \cref{app:other_figs}.} 
For acceptability judgments, also in English, we use the Corpus of Linguistic Acceptability (CoLA; \citealt{warstadt2018neural}) and the \defn{\textsc{bnc}} dataset \cite{lau-2017-grammaticality}. Notably, Natural Stories and CoLA by design contain wide coverage of syntactic and semantic phenomena. We provide further details of each of these datasets, including pre-processing, statistics and data-gathering processes, in \cref{app:data}.

\subsection{Estimating Surprisal}
Since we do not have access to the ground-truth values of conditional probabilities of observing linguistic units given their context (i.e., surprisals), we must instead estimate these probabilities. This is typical practice in psycholinguistic studies  \cite{DEMBERG2008193,mitchell-etal-2010-syntactic,fernandez-monsalve-etal-2012-lexical}. For example, \citet{hale-2001-probabilistic} uses a probabilistic context-free grammar; \citet{smith2013-log-reading-time} use $n$-gram language models.\clara{would have been interesting to test this stuff out with parse steps, like with a PCFG} 

In general, the psychometric predictive power of surprisal estimates from a model correlates highly with model quality  \citep[as traditionally measured by perplexity;]{frank-bod,fossum-levy:2012cmcl,goodkind-bicknell:2018predictive}. Further, Transformer-based models appear to have superior psychometric predictive power in comparison to other architectures \cite{wilcox2020}. We employ GPT-2 \cite{radford2019language}, TransformerXL \cite{dai-etal-2019-transformer}, and BERT \cite{devlin-etal-2019-bert}---state-of-the-art language models\footnote{Notably, BERT is a cloze language model. Thus, the probabilities it provides are \emph{pseudo} surprisal estimates.}. We additionally include results using a 5-gram model, estimated using Modified Kneser--Essen--Ney Smoothing \cite{NEY19941}, to allow for an easier comparison with results from earlier works exploring UID in reading time data. All probability estimates are computed at the word-level.\footnote{
Given the hierarchical structure of language, there is not a single ``correct'' choice of linguistic unit over which language processing should be analyzed.
Here we consider the primary units in a linguistic signal to be words, where we take a sentence to be a complete linguistic signal.  
%for example, for the reading time datasets that we employ, all metrics are reported at the word-level. 
We believe similar analyses at the morpheme, subword or phrase level---which we leave for future work---may shed further light on this topic.\looseness=-1} Further details are given in \cref{app:data}.\looseness=-1

\subsection{Assessing Predictive Power}\label{sec:assess}

In our experiments, we analyze the ability of different functions of surprisal to predict psychometric data, namely the total time spent reading sentences in self-paced reading and eye tracking studies (see \cref{app:data})---and perceived linguistic acceptability,\footnote{Language models are trained to predict the probability of a sentence; the concept of linguistic acceptability is not explicitly part of their objective. As such, probability under a language model alone does not necessarily correlate well with acceptability \cite{lau-2017-grammaticality}.} in order to better understand the relationship of surprisal with language processing. For reading times, we use the sum across word-level times as our sentence-level metric. 
Notably for eye movement datasets, our analysis of sentence-level reading times is novel: previous work has generally focused on how long readers spend on a word \emph{before} progressing beyond it (often called the ``first pass;'' \cite{rayner:1998}), but sentence-level measures include time re-reading content after having progressed beyond it. Linguistic acceptability data are available and assessed only at the sentence-level.\looseness=-1
 
\begin{figure*}[]
    \centering
    \includegraphics[width=\textwidth]{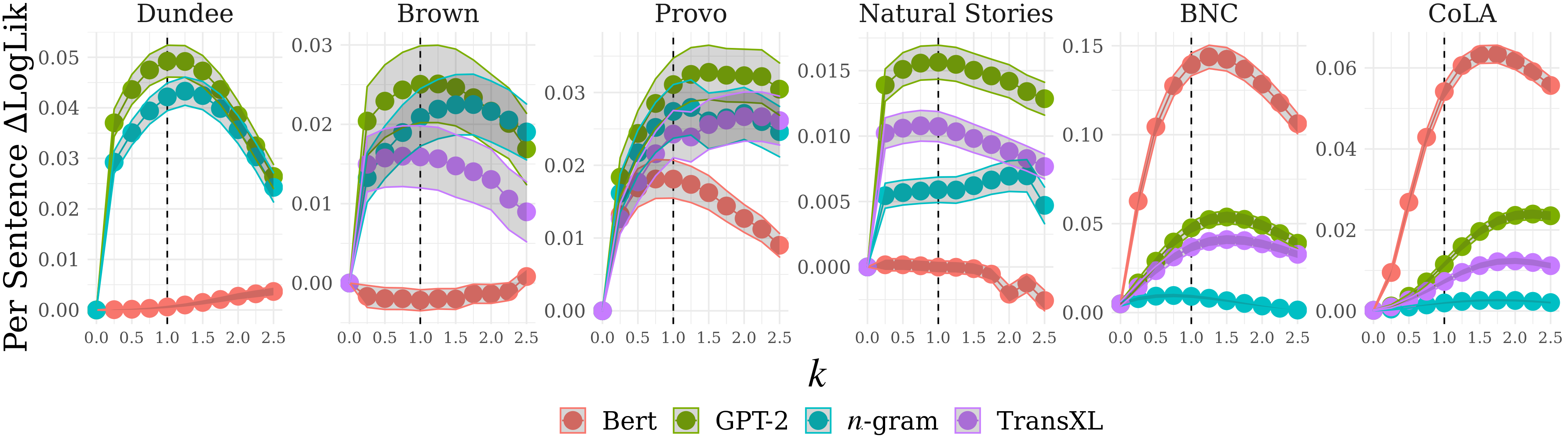}
    \caption{Mean $\deltaLL$ as a function of the exponent $k$ for the sentence-level predictor (\cref{eq:cost}) of reading time and linguistic acceptability. Shaded region connects standard error estimates from each point. We observe that often, our predictor with $k>1$ explains the data at least as well as $k=1$. Baseline models against which  $\deltaLL$ is computed are specified in \cref{sec:assess}. For reading times, the augmented models additionally contain fixed effects and per-subject random effects slopes for the UID operationalization; for acceptability judgments, only a fixed effect for the UID operationalization is added. %For reading time predictions, baseline predictors include sentence length (in words), number of words with recorded fixations (per subject-sentence), and intercept random effects (per subject), i.e., directly mimicking \cref{eq:cost}. For acceptability judgments, there are no additional predictors. 
    \looseness=-1 }
    \label{fig:all}
\end{figure*}

As we are interested in the relationship between UID and both reading times and acceptability judgments---in particular, the relationships described by \cref{eq:peffort2,eq:accept}---we turn to linear regression models.\footnote{While, for example, a multi-layer perceptron may provide more predictive power given the same variables, we may not be able to interpret the learned relationship as additional transformations of our independent variables would likely be learned. Using linear regression allows us to directly assess which functions of surprisal more accurately explain data under our linearity assumptions in \cref{eq:peffort2,eq:accept}.\looseness=-1} 
For reading time data, as our baseline models, we specifically use linear mixed-effects models, with random effect terms (slopes for total word count at the sentence-level and intercepts at the word-level) for each subject to control for individual reading behaviors.\footnote{Mixed-effects models allow us to incorporate both fixed and random effects into the modeling process, helping bring the conditional independence assumptions of the regression analysis better in line with the grouping structure of repeated-measures data.\looseness=-1} We additionally control for other variables known to influence reading time: at the sentence-level, our fixed effects include total word count and number of words with recorded fixations (per subject and sentence);\footnote{In natural reading, some words are never fixated (so-called skips). Hence, we include the number of fixated words in addition to actual sentence length.} results including fixed effects for sums of both individual word character lengths and word unigram log-probabilities (as estimated from WikiText 103; \citealt{wikitext103}) are given in \cref{app:other_figs}. At the word-level (only our last set of experiments), our fixed effects include linear terms for word log-probability, unigram log-probability, and character length, and the interaction of the latter two. We additionally include the same predictors from the previous word, a common practice due to known spillover effects observed in both types of measurement. These are standard predictors in reading time analyses \cite{smith2013-log-reading-time,goodkind-bicknell-2018-predictive,wilcox2020}. For linguistic acceptability data, we use logistic regression models with solely an intercept term as our baseline predictor; results when including summed unigram log-probability or sentence length as predictors yielded similar trends (see \cref{app:other_figs}).\looseness=-1

We evaluate each model relative to a baseline, containing only the control features just mentioned. Specifically, performance assessments are computed between models that differ by solely a single predictor; for reading time data, we include both a fixed and (per-subject) random slope for this predictor. Following \citet{wilcox2020}, we report $\deltaLL$: the mean difference in log-likelihood of the response variable between the two models. A positive $\deltaLL$ value indicates that a given data point is more probable under the comparison model, i.e., it more closely fits the observed data. To avoid overfitting, we compute $\deltaLL$ solely on held-out test data, averaged over 10-fold cross validation. See \cref{app:data} for evaluation details.

\subsection{Results}\label{sec:exps}

\paragraph{Evidence of UID in Reading Times and Acceptability Judgments.}
We first assess the ability of our processing cost model (\cref{eq:cost}) to predict reading times. In a similar fashion, we use \cref{eq:accept} with \cref{eq:power} to predict acceptability scores. Recall from \cref{sec:background} that if the true relationship between surprisal and sequence-level processing effort is expressed by \cref{eq:cost} with $k>1$, then there must exist a pressure towards uniform information density. Thus, if we observe that a linear model using $\sum_{n=1}^N s(u_n)^k$ as a predictor explains the observed data better when $k> 1$, it suggests a preference for the uniform distribution of information in text.\footnote{This of course is under the assumption that when $k>1$, the coefficient for the term is positive for reading time, i.e., higher values correlate with longer reading time, and negative for acceptability judgments, i.e., higher values correlate with lower acceptability scores. Notably, the opposite logic holds for $k<1$: we would expect coefficients to be flipped if it provides better predictive power than $k=1$.}
We report results for multiple corpora in \cref{fig:all}.\footnote{We also perform experiments using additional predictors and on the Dutch GECO corpus, finding consistent results. See \cref{app:other_figs}.\looseness=-1}

We see that in general, the best fit to the data is achieved not when our cost equations use $k=1$, but rather a slightly larger value of $k$ (see also \cref{tab:sentence-level}). Notably for reading time data, a conclusion that $k>1$ is optimal contradicts a number of prior works that have judged the relationship between surprisal and reading time to be linear. 
%Yet given how close to $1$ the optimal values of $k$ typically are for reading time datasets, this is perhaps not surprising; under most experimental setups, a slight super-linear relationship would have been virtually undetectable. 
We discuss this point further in \cref{sec:discussion}. Yet for the reading time datasets, the $k=1$ predictor is typically still within the standard error of the best predictor, meaning that the linear hypothesis is not ruled out. For acceptability data, we see more distinctly that $k>1$ leads to the best predictor, especially when using true surprisal estimates (i.e., models aside from BERT). This result suggests that a more uniform distribution of information more strongly correlates with linguistic acceptability (see also \cref{fig:cola} for explicit correlation analysis).

We perform hypothesis tests to formally test whether our models of processing cost and linguistic acceptability have higher predictive power---as measured by $\deltaLL$---when using a super-linear vs. linear function of surprisal. Specifically, we take our null hypothesis to be that $k=1$ provides better or equivalent predictive power to $k>1$. We use a paired t-tests, where we aggregate sentence-level data across subjects for reading time datasets so as not violate independence assumptions. We use a Bonferroni correction to account for the consideration of multiple models with $k>1$. We find that we consistently reject the null hypothesis at significance level $\alpha = 0.001$ for acceptability data experiments (aside from under the $n$-gram model). For reading time data, we never reject the null hypothesis, again confirming that the linear hypothesis may hold true in this setting.

Another important observation is that the pseudo log-probability estimates from a cloze language model (BERT) work remarkably well when used to predict acceptability judgments, yet remarkably poorly for reading time estimates. We also see a less super-linear effect (higher predictive power for $k \approx 1$) of surprisal in sentence acceptability for cloze than for auto-regressive models.\footnote{\citet{schrimpf2020neural} found GPT-2 superior to BERT for encoding models to predict brain response during language comprehension. We leave further exploration of the general issue for future work.\looseness=-1}\looseness=-1
%As previously mentioned, a number of prior works have concluded that the effects of surprisal on cognitive load are linear. Accordingly, we may expect that setting $k=1$ would lead to the most powerful predictor for our data. However, in \cref{fig:all},
\begin{table*}[!t]
  \centering
  \small
  \begin{adjustbox}{max width=\textwidth}
  \begin{tabular}{ lrrrrrr }
   \toprule
  \multirow{2}{*}{{\bf Predictor}} &  \multicolumn{4}{c}{Reading Time} & \multicolumn{2}{c}{Acceptability} \\
      & \multicolumn{1}{c}{Dundee} &\multicolumn{1}{c}{Brown} & \multicolumn{1}{c}{Provo}  &\multicolumn{1}{c}{NS} &  \multicolumn{1}{c}{CoLA} & \multicolumn{1}{c}{\textsc{bnc}} \\
     \hline\\[-1.5ex]
   $\mathrm{Super}\text{-}\mathrm{Linear}$ ($k=0.25$)	& $	3.70\db{0.27}	$&$	1.88\db{0.44}	$&$	1.73\db{0.27}	$&$	1.40\db{0.12}	$&$	0.90\db{0.03}	$&$	6.11\db{0.13}	$\\
$\mathrm{Super}\text{-}\mathrm{Linear}$ ($k=1$)	& $	\bf{4.93}\db{0.32}	$&$	2.38\db{0.48}	$&$	3.07\db{0.36}	$&$	\bf{1.58}\db{0.13}	$&$	5.28\db{0.07}	$&$	13.89\db{0.19}	$\\
$\mathrm{Super}\text{-}\mathrm{Linear}$ ($k=1.25$)	& $	\bf{4.93}\db{0.31}	$&$	\bf{2.39}\db{0.49}	$&$	3.24\db{0.37}	$&$	1.55\db{0.13}	$&$	5.92\db{0.07}	$&$	\bf{14.35}\db{0.19}	$\\
$\mathrm{Super}\text{-}\mathrm{Linear}$ ($k=1.5$)	& $	4.74\db{0.31}	$&$	2.34\db{0.49}	$&$	\bf{3.25}\db{0.37}	$&$	1.50\db{0.13}	$&$	\bf{6.18}\db{0.07}	$&$	14.22\db{0.19}	$\\
$\mathrm{Super}\text{-}\mathrm{Linear}$ ($k=2$)	& $	3.85\db{0.28}	$&$	2.11\db{0.47}	$&$	3.22\db{0.36}	$&$	1.40\db{0.13}	$&$	6.04\db{0.07}	$&$	12.75\db{0.18}	$\\
$\mathrm{Variance}$ (lang)	& $	2.37\db{0.22}	$&$	1.37\db{0.39}	$&$	2.46\db{0.33}	$&$	0.73\db{0.10}	$&$	5.64\db{0.07}	$&$	11.26\db{0.17}	$\\
$\mathrm{Variance}$ (sent)	& $	2.01\db{0.20}	$&$	1.16\db{0.35}	$&$	2.59\db{0.34}	$&$	0.80\db{0.11}	$&$	1.86\db{0.04}	$&$	7.56\db{0.14}	$\\
$\mathrm{Local Variance}$	& $	1.93\db{0.20}	$&$	1.08\db{0.36}	$&$	2.15\db{0.30}	$&$	0.64\db{0.09}	$&$	1.44\db{0.04}	$&$	4.88\db{0.12}	$\\
$\mathrm{Max}$	& $	1.74\db{0.20}	$&$	1.11\db{0.39}	$&$	1.17\db{0.27}	$&$	0.68\db{0.12}	$&$	1.16\db{0.03}	$&$	5.00\db{0.12}	$\\
$\mathrm{Entropy}$ ($k=0.25$)	& $	1.16\db{0.16}	$&$	0.30\db{0.24}	$&$	1.35\db{0.22}	$&$	0.25\db{0.13}	$&$	0.02\db{0}	$&$	0.03\db{0.01}	$\\
\makecell[l]{$\mathrm{Entropy}$ ($k=1$) {\scriptsize Shannon}}	& $	-0.01\db{0.01}	$&$	0\db{0}	$&$	0.01\db{0}	$&$	0\db{0}	$&$	0\db{0}	$&$	7.90\db{0.14}	$\\
\makecell[l]{$\mathrm{Entropy}$ ($k=2$) {\scriptsize  Renyi}}	& $	-0.01\db{0}	$&$	0\db{0.01}	$&$	0\db{0.01}	$&$	0\db{0}	$&$	0\db{0}	$&$	8.38\db{0.14}	$\\
    \bottomrule
  \end{tabular} 
  \end{adjustbox}
  \caption{ $\deltaLL$ in $10\text{e-}2$ nats when adding different UID operationalizations as predictors of reading time and linguistic acceptability. Surprisal estimates from GPT-2 are used. We use the same paradigm for baseline and augmented models as in \cref{fig:all}. Other setups show similar trends (\cref{app:other_figs}).\looseness=-1}
  \label{tab:sentence-level}
\end{table*}\roger{This table and Figure 3 still don't look totally aligned, eg Provo is optimal at k=2 in the table but in the figure it looks like it's k=1.75?}\clara{yes, but k=1.75 isn't included in the table :/ I skipped it for space. Do you think I should add it in?}

\paragraph{Evaluating Operationalizations of UID.}
We next ask: what are appropriate measures of UID in a linguistic signal? In an effort to answer this question, we explore the predictive power of the different operationalizations of UID proposed in \cref{sec:sequence-level} for our psycholinguistic data; given our evidence of UID in the prior section, we posit that better operationalizations should likewise provide stronger explanatory power than poor ones. We again fit linear models using \cref{eq:peffort2,eq:accept}, albeit with each analyzed UID operationalization as our predictor. We use surprisal estimates from GPT-2, as it was consistently the autoregressive language model with the best predictive power.\looseness=-1

Results in \cref{tab:sentence-level} show that, in general, the family of $\mathrm{Super}\text{-}\mathrm{Linear}$ (\cref{eq:power}) operationalizations (for $k\geq 1$) and a language-wide notion of $\mathrm{Variance}$ (\cref{eq:var}) provide the largest increase in explanatory power relative to the baseline models, suggesting they may be the best quantifications of UID. While the $\mathrm{Max}$ (\cref{eq:max}) and $\mathrm{Variance}$ (\cref{eq:var}) predictors also provide good explanatory power, they are consistently lower across datasets. %[which has been analyzed in a number of other studies;%This juxtaposition provides modest support for uniform information density (more so than surprisal theory) as an explanation of cognitive processing load.
Further, language-level $\mathrm{Variance}$ seems to produce stronger predictors for psychometric data than sentence-level and $\mathrm{Local\,Variance}$---an observation driving our next set of experiments. 
%---an interesting result deserving further exploration by future work.
Notably, the $\mathrm{Entropy}$ predictors do quite poorly in comparison to other operationalizations, especially for $k\geq 1$.\footnote{While this could be attributed to the artificial normalization of $s(u_1), \dots, s(u_n)$ that must occur to generate a valid probability distribution, we saw similar trends when using the original, unnormalized distribution $s(u_1), \dots, s(u_N)$.} These results suggest that a sentence-level notion of entropy may not capture the UID phenomenon well, which is perhaps surprising, given that it is a natural measure of the uniformity of information.
%From these results, we conclude there are a number of viable operationalizations

\paragraph{Exploring the Scope of UID's Pressure.}\roger{Isn't this paragraph now outdated/superseded given the new Table 1 results? Should we just ditch it?}\clara{I don't think its outdated, actually. The results in Table 1 are not too different than they were before. Variance was never the best operationalization but its the operationalization that gives us the chance to ask this question}
Each of our operationalizations in \cref{sec:sequence-level} are computed at the sequence-level. Thus, it is natural to ask, what should be the scope of a sequence when considering information uniformity?  
% Is the pressure for uniformity at the phrase, sentence, document, or perhaps even language level?
In an effort to answer this question, we explore how the predictive power of our UID operationalizations change as we vary the window sizes over which they are computed. 
Specifically, we will look at ability to predict per-word reading times; we make use of the $\mathrm{Variance}$ operationalization as our predictor (which demonstrated good performance in our sentence-level experiments) albeit with a word-level version:
\begin{equation}\label{eq:var_word}
    \uidinv(u_n) = (\surp(u_n) - \mu)^2
\end{equation}
where $\mu$ is mean surprisal computed across the previous $1,2,3,4$ or $n$ words or across the sentence, document, or language as a whole (as with unigram probabilities, $\mulang$ is computed per model over WikiText 103). 
%By observing how predictive power changes in each of these settings, we hope to gain intuition for the scope over which UID holds. 
% 
\cref{tab:sentence-level} and \cref{fig:windows} show evidence that the pressure for uniformity may in fact be at a more global scale. Under each corpus, the higher-level predictors of UID appear to provide better explanatory power of reading times than more local predictors.\looseness=-1  
\begin{figure}
    \centering
    \includegraphics[width=\linewidth]{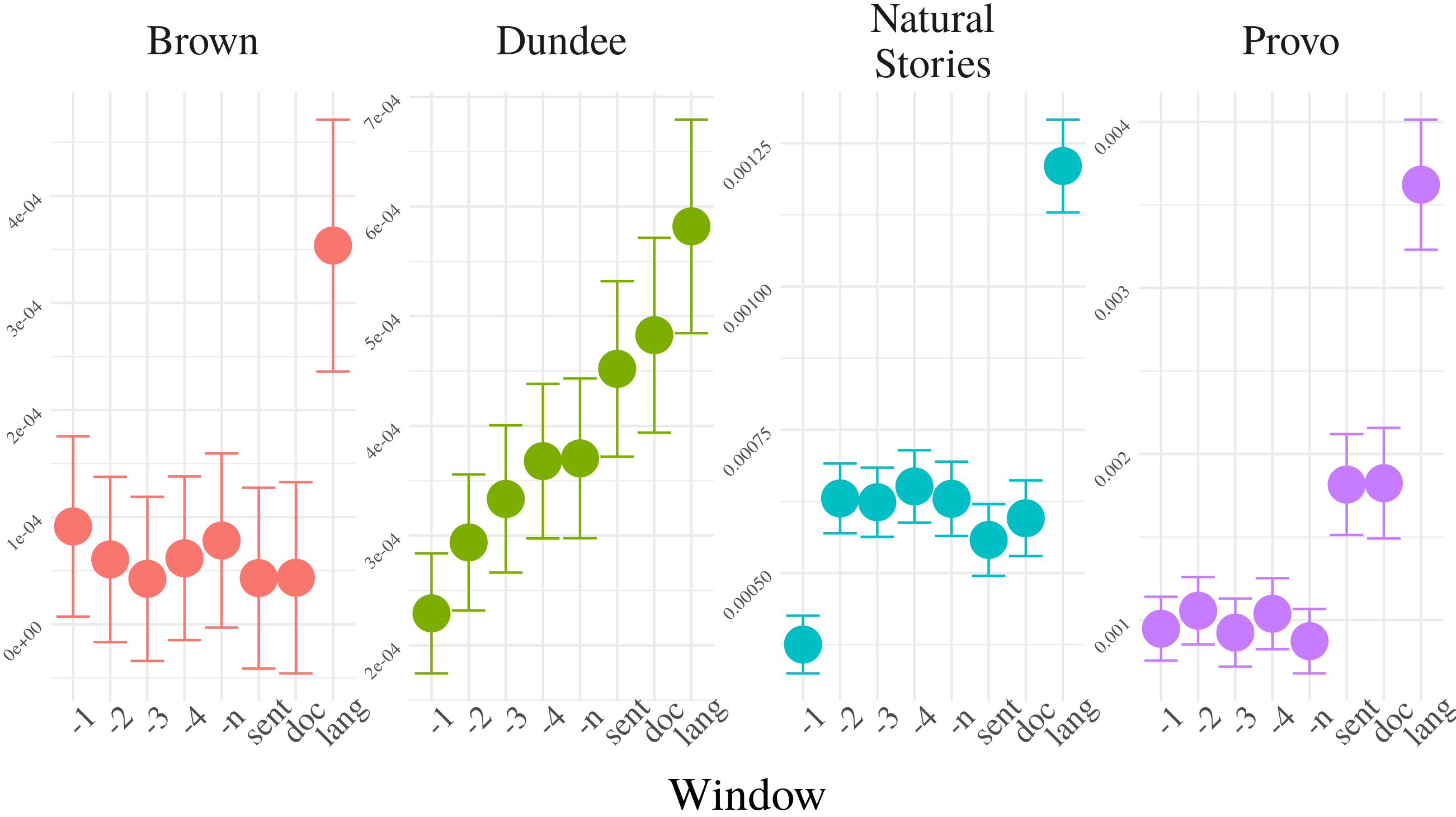}
    \caption{Per-token $\deltaLL$ when changing the scope over which UID variance is computed (see \cref{eq:var_word}). Surprisal estimates from GPT-2 are used. Baseline predictors are specified in \cref{sec:assess}} \label{fig:windows}
\end{figure}

\section{Discussion }\label{sec:discussion}
Most previous works investigating UID have looked for its presence in language production \cite[\textit{inter alia}]{bell2003effects,aylett-2004,levy2007speakers,MAHOWALD2013313}, while comprehension has received little attention.
\citet{collins2014} and \citet{sikos2017information} are perhaps the only other works to find results in support of UID in this setting. 
Our findings are complementary to theirs; we take different analytical approaches but both observe a preference for the uniform distribution of information in a linguistic signal, although a similar analysis should be performed in the spoken domain before stronger conclusions can be drawn. 

While our reading time results do not refute previous work showing linear effects of surprisal on word-level reading times \cite{smith2013-log-reading-time,goodkind-bicknell-2018-predictive, wilcox2020},\footnote{Notably, \citet{BROTHERS2021104174} have recently reported a \emph{linear} effect of word probability on (self-paced) reading times in a controlled experiment where within each experimental item the target word was held constant and predictability was manipulated across a wide range by varying the preceding context. Motivated by this result, we repeated our analytic pipeline testing a range of values of $k$ but replacing surprisals with negative \emph{raw} probabilities. The resulting regression model fits are not as good as those achieved when using surprisals (\cref{fig:probs}; compare $y$-axis ranges with \cref{fig:all}).\looseness=-1} we see some suggestions that a super-linear hypothesis is also plausible, especially in the Provo corpus.
 Notably, most of these works did not test a parametric space of non-linear functional forms, instead confirming using visual inspection of the results of nonparametric fits. One exception, \citet{smith2013-log-reading-time}, 
%do compare linear vs. super-linear relationships; specifically, they 
explored the effects of adding a quadratic term for surprisal as a predictor of per-word reading times. Yet, if the true $k$ that describes the reading times--surprisal relationship were only slightly greater than 1, as our results suggest, this quadratic test might be too restrictive. 
%Indeed we see in our results that a $k$ only slightly above 1 close to (albeit greater than) 1, as we observe in our own experiments, a linear term may capture most variation in the data; the addition of a second quadratic term may then appear to have an insignificant effect on predictive power.
Our approach, which explores a more fine-grained range of $k$, is potentially more comprehensive, and indeed we find that values of $k$ slightly greater than $1$ often fit the data at least as well as $k=1$, and can certainly not be ruled out. 
Other potential virtues of our analysis are (1) Our analysis is performed at the sentence- (rather than word-) level. This is arguably a better method for analyzing a sequence-level phenomenon, i.e., UID, 
%at which additional cognitive processes such as sentence- or clause-level discourse integration can incur costs 
and (2) specifically for eye movement data, we include re-reading times after the first pass.\looseness=-1

\paragraph{Limitations and Future Directions.}
A major limitation of this work is that the experimental analysis is limited to English (and Dutch, in the Appendix); while the pressure for uniformity---since explained by a cognitive process---should hold across languages, further experiments should be performed to verify these findings, especially since the relationship between model quality and psychometric predictive power has recently been called into question \cite{kuribayashi-etal-2021-lower}. As such, while we find convincing preliminary evidence in our analyzed languages, we are not able to fully test the hypothesis that the pressure for UID is at the language-level. Further, we have no evidence as to whether there may be pressure towards a \emph{cross-linguistic} $\mu_c$, which would be relevant to cross-linguistic interpretations of UID \cite{pimentel+al.emnlp2021a}.

Another important limitation of this work is the restriction to psychometric data from the written domain. To fully grasp the effects of the distribution of information in linguistic signals on language comprehension, spoken language data should be similarly analyzed. Of course, different factors are likely at play in language comprehension in the spoken domain, including e.g., the cognitive load of the speaker \cite{pijpops2018comparing}; such factors may make it even more difficult to disentangle the contribution of different effects to comprehension. We leave this analysis for future work.\looseness=-1

\section{Conclusion}
In this work, we revisit the UID hypothesis, providing both a quantitative and qualitative assessment of its various interpretations. We find suggestions that the UID formulation proposed in \citet{levy-thesis} may better predict processing effort in language comprehension than alternative formulations since proposed. We additionally find that a similar model explains linguistic acceptability judgments well, confirming a preference for UID in written language. We subsequently evaluate different operationalizations of UID, observing that a super-linear function of surprisal best explains psychometric data. Further, operationalizations associated with global interpretations of UID appear to provide better explanatory power than those of local interpretations, suggesting that perhaps the most accurate interpretation of UID should be the regression towards the mean information rate of a language.

\section*{Acknowledgments}
We thank our anonymous reviewers, who provided invaluable feedback on the manuscript for this work. Lena J\"ager was partially funded by the German Federal Ministry of Education and Research under grant 01\textpipe S20043. RPL acknowledges support from the MIT–IBM AI Lab, the MIT Quest for Intelligence, and NSF award BCS-2121074.

\bibliographystyle{acl_natbib}
\bibliography{anthology,acl}

\clearpage
\newpage

\appendix
% \onecolumn
\section{Theory}\label{sec:unique-min-proof}

We use the standard definition of surprisal $\surp(u_n) \defeq - \log p(u_n \mid \mathbf u_{< n})$, and define $\surp (\ww)= \sum_{n=1}^N s(u_n)$ as the total surprisal of the entire signal $\ww$.\looseness=-1

% \mathcheck{\begin{equation}\label{eq:uid_trade_off}
%     \left[ \frac{\sum_{n=1}^N  s(u_{n})}{N}  \right]^k \leq \frac{1}{N} \sum_{n=1}^N \left[ s(u_{n}) \right]^k
% \end{equation}} 
\begin{theorem}
Assume a fixed $k > 1$ and $c > 0$, and assume $N \geq 1$. Then,
\begin{enumerate}[label=\roman*)]
\item The objective $\sum_{n=1}^N s(u_{n})^k + c\cdot N$, i.e. \cref{eq:cost}, subject to the constraint of a fixed  $s(\ww) = \sum_{n=1}^N s(u_n)$, is minimized when information is uniformly distributed, i.e. $s(u_1) = s(u_2) = \cdots = s(u_N) = s(\ww)/N$;
\item Furthermore, this minimal value is found for either one or two choices of finite $N$.
\end{enumerate}
\end{theorem}
\begin{proof}
We prove i) and ii) separately.
%(I) we show that for any message length $N$, the uniform distribution of surprisals minimizes processing cost; (II) we show that among messages with uniformly distributed surprisals, there exist either one or two message lengths $N$ that minimize processing cost.

\paragraph{i).} This was proven in the Appendix of \citet{levy2007speakers} as a simple application of Jensen's inequality, which we reproduce here in largely similar form (adapting to our notation). 
First note that the function $(\cdot)^k$ is convex on the interval $[0,\infty)$ for $k > 1$; as surprisal can only take on positive values, this is the interval we operate over. 
Since $\sum_{n=1}^N \frac{1}{N} = 1$ and $\frac{1}{N} \geq 0$, we have that
$\sum_{n=1}^N \frac{\surp(u_n)^k}{N}$ is a convex combinations of the exponentiated surprisals $\surp(u_n)^k$.
Thus, as we have a convex combination of convex functions, we may invoke Jensen's inequality, which yields
\begin{equation}
  \sum_{n=1}^N \frac{\surp(u_n)^k}{N} \geq \left(\frac{s(\ww)}{N} \right)^k
%   \sum_{n=1}^N \frac{1}{N} \surp(u_n)^k \geq \left(- \log
%     {p(\ww)^\frac{1}{N}} \right)^k
\end{equation}
Multiplying both sides by $N$ gives
\begin{equation}\label{eq:final-result}
  \sum_{n=1}^N \surp(u_n)^k \geq N \left(\frac{s(\ww)}{N} \right)^k
%   \sum_{n=1}^N \frac{1}{N} \surp(u_n)^k \geq \left(- \log
%     {p(\ww)^\frac{1}{N}} \right)^k
\end{equation}
The lower bound of \cref{eq:final-result} tells us that uniformly distributed information, i.e. where each $s(u_n) = s(\ww) / N$  is the lowest cost manner to distribute total surprisal over the utterance. 
Conversely, when $0<k<1$, $(\cdot)^k$ is concave on the interval $[0,\infty)$. Therefore, the same logic gives us the opposite result: Uniform information density is the \emph{highest} possible cost way to distribute total surprisal over the utterance.\looseness=-1

\paragraph{ii).} %We first show that \cref{eq:cost} is convex in message length $N$. 
As shown in the previous step, regardless of the value of $N$, $\peffort$ is minimized when information density is uniform---that is, when $s(u_{n})=s(\ww)/N$---giving us:\ryan{Do we want to say what we plugged $s(u_{n})=s(\ww)/N$ into to get the equation below? It comes out of nowhere!}
\begin{subequations}
\begin{align}
\peffort & = N \left[\frac{s(\ww)}{N}\right]^k + c\cdot N  \\
& = \frac{s(\ww)^k}{N^{k-1}} + c\cdot N.
%\\& = B^k\,N^{-(k-1)} + c\cdot N .
\end{align}
\end{subequations}
We now consider the question of what value of $N$ minimizes $\peffort$. A continuous extension of $\peffort$ to real-valued $N$ has the following first and second derivatives:
\begin{subequations}
\begin{align}\label{eq:first-derivative}
\frac{\partial \peffort}{\partial N} &= -(k-1) \frac{s(\ww)^k}{N^{k}} + c \\
\frac{\partial ^2\peffort}{\partial N^2} 
&= k(k-1) \frac{s(\ww)^k}{N^{k+1}}\label{eq:ssecond-derivative}
\end{align}
\end{subequations}
We can use these derivatives to inspect the behavior of the function.
First, the second derivative is strictly positive, thus processing effort is strictly convex in $N$ so it has at most one global minimum.
Second, we can find the minimizing value of $N$ by setting the first derivative to zero, giving us:\ryan{I think this might still not be right. Shouldn't the root $k$ cancel with the $s(u_n)^k$? Seems like an algebra mistake. (Roger agrees, and has implemented this correction -- RPL)}\tiago{Yeah, it was wrong, my bad.}
\begin{align}
N^\star = \left(\frac{k-1}{c} \right)^{\frac{1}{k}} s(\ww)
\end{align}
However, since this is a constrained optimization problem ($N \geq 1$), we arrive at the solution
\begin{equation}
N^\star = \max \left(1, \left(\frac{k-1}{c} \right)^{\frac{1}{k}} s(\ww)\right)
\end{equation}
which is true because the first derivative will be strictly positive for any value of $N$ above its global minimum $\left(\frac{k-1}{c} \right)^{\frac{1}{k}} s(\ww)$.
Now, to address the finiteness of $N^\star$, we observe that as $N \rightarrow \infty$, we have $\frac{\partial \peffort}{\partial N} \rightarrow c > 0$ so the function cannot achieve its minimum as $N \rightarrow \infty$.\tiago{Is this sentence necessary after we have the value of N* and the fact that the second derivative is strictly positive?} 
Returning to integer-valued $N$, we have that processing effort is minimized either at floor($N^\star$), ceiling($N^\star$), or both.
Finally, it is important to highlight that if the first derivative (i.e., \cref{eq:first-derivative}) is positive at $N=1$, we arrive at the result that processing effort is minimized at $N=1$. 
This will happen when $s(\ww)$ is sufficiently small and/or $c$ is sufficiently large: the amount of information to be communicated is not worth the cost of using more than a minimal-length utterance.

Note also that for $0 < k < 1$, when $(\cdot)^k$ is concave, we obtain a different, and counter-intuitive result: the first derivative is \emph{always} positive, meaning that processing effort is minimized at $N=1$ regardless of $s(\ww)$ or $c$.

\end{proof}

\section{Datasets and Language Models }\label{app:data}
\begin{table}
  \centering
  \adjustbox{max width=\linewidth}{
  \begin{tabular}{lrrrrr}
   \toprule
    % &  \multicolumn{4}{c}{\bf \# of} \\
    {\bf Dataset } & Types (\textsc{m})  & Types (\textsc{u}) & Sents (\textsc{m}) & Sents (\textsc{u}) & Docs (\textsc{u})   \\
     \hline
 \makecell[l]{Natural\\Stories} & 848,852&10,256 & 41,788&485 & 10\\
 Provo & 225,624& 2,745 &11,340 & 2,689& 55\\
 Dundee &614,689 & 51,501 & 23,777 &2,377 & 20\\
 Brown &547,628 & 7,234 & 34,284 & 1,800 & 13\\
 CoLA &- &65,809 & 10,657& 10,657& -\\
 \textsc{bnc} & - & 43,318 & 2,500 & 2,500 & -\\
 \bottomrule
 
  \end{tabular} }
  \caption{ Dataset statistics. \textsc{u} refers to \emph{unique} counts while \textsc{m} refers to \emph{measured} counts, i.e. number of collected data points. }
  \label{tab:data}
\end{table}
\paragraph{Data pre-processing.} Text from all corpora was pre-processed using the Moses decoder\footnote{\url{http://www.statmt.org/moses/}} tokenizer and punctuation normalizer. Additional pre-processing was performed by the Hugging Face tokenizers for respective neural models. Capitalization was kept intact albeit the lowercase version of words were used in unigram probability estimates. We estimate the unigram distribution following \citet{nikkarinen+al.acl-findings2021}. Sentences were delimited using the NLTK sentence tokenizer.\footnote{\url{https://www.nltk.org/api/nltk.tokenize.html}} For reading time datasets, we removed outlier word-level data points (specifically those with a $z$-score $>3$ when the distribution of reading times was modeled as log-linear). We omitted the sentence-level reading time for a specific subject from our analysis if it contained any outlier data points.

The \defn{Natural Stories} consists of a series of English texts that were hand-edited to contain low-frequency syntactic constructions while still sounding fluent to native speakers. It contains 10 stories with a total of 485 sentences. Self-paced reading data from these texts was collected from 181 native English speakers. The appeal of this corpus lies in that it provides psychometric data on unlikely---but still grammatically correct---sentences, which in theory should provide broader coverage of the sentence processing spectrum. 

The \defn{Provo Corpus} consists of 55 paragraphs of English text (with a total of 2,689 sentences) taken from various sources and genres, including online news articles, popular science, and fiction. Eye movement data while reading from 84 native speakers of American English was collected using a high-resolution eye tracker (1000 Hz). We specifically use the \textsc{ia-dwell-time} attribute as our measure of per word reading time; specifically, we use the summation of the duration across all fixations on that word. We find noisier trends when using \textsc{ia-first-run-dwell-time} and \textsc{ia-first-fixation-duration} (see \cref{app:other_figs}). %\clara{Need to talk about metric used: IA-DWELL-TIME}

The English portion of the \defn{Dundee Corpus} contains eye-tracking recordings (1000 Hz) of 10 native English-speakers each reading 20 newspaper articles from \emph{The Independent}, with a total of 2,377 sentences. Unlike in previous studies (e.g. \citet{goodkind-bicknell-2018-predictive}) we did not exclude any words from the dataset, as we were interested in sentence-level measures. As with the Provo corpus, we use total dwell time as our dependent variable.

The \defn{Brown Corpus} consists of self-paced reading data for selections from the Brown corpus of American English. Moving-window self-paced reading times were measured for 35 UCSD undergraduate native English speakers, each reading short (292–902 word) passages drawn from the Brown corpus of American English (total of 1,800 unique sentences). Data from participants were excluded if comprehension–question performance was at chance. Further details about the procurement of the dataset are described in \cite{smith2013-log-reading-time}.

The Dutch portion of the \defn{GECO}---Ghent Eye-Tracking Corpus---contains eye-tracking recordings from bilingual (Dutch/English) participants reading a portion of a novel, presented in paragraphs on the screen. 

For  CoLA, sentences are taken from published linguistics literature and labeled by expert human annotators. According to the authors, ``unacceptable sentences in CoLA tend to be maximally similar to acceptable sentences and
are unacceptable for a single identifiable reason,'' which implies that differentiability should be nuanced rather than, e.g.,  from a blatant disregard for grammaticality. 
We also utilize the \defn{\textsc{bnc}} dataset \cite{lau-2017-grammaticality}, which consists of 2500 sentences taken from the British National Corpus. Each sentence is round-trip machine-translated and the resulting sentence is annotated with acceptability judgments through crowd-sourcing. Two rating systems are provided for this corpus: \textsc{mop2} and \textsc{mop4}. The former provides binary judgments of acceptability while the latter provides a score from 1-4. We employ the former in our predictive power experiments so as to share the same setup for the CoLA dataset; we use the latter in computations of correlation. 

For probability estimates from neural models, we use  pre-trained models provided by Hugging Face \cite{wolf-etal-2020-transformers}. Specifically, for GPT-2, we use the default OpenAI version (\texttt{gpt2}). The model was trained on the WebText dataset (a diverse collection of approximately 8 million websites); it uses byte-pair encoding \citep{sennrich2015neural} with a vocabulary size of 50,257. For the TransformerXL, we use a version of the model (architecture described in \citet{dai-etal-2019-transformer}) that has been fine-tuned on WikiText-103 (\texttt{transfo-xl-wt103}). We use the \texttt{bert-base-cased} version of BERT. In all cases, per-word surprisal is computed as the sum of subword surprisals. We additionally train a 5-gram model on WikiText-103 using the KenLM \cite{kenlm} library with default hyperparemters for Kneser--Essen--Ney smoothing.

\paragraph{Evaluation.}
For our evaluation metric, we use $\deltaLL$: the mean difference in log-likelihood of the response variable between a baseline model and a model with an additional predictor. A positive $\deltaLL$ value indicates that a given data point is more probable under the comparison model, i.e., the comparison model more closely fits the observed data. To compute $\deltaLL$ for each data point, we split our corpus into 10 folds. Folds are chosen randomly, i.e., they are not based on subject or sentence for mixed-effects models. The same splits are used for each model. We take the $\deltaLL$ value for a data point to be the difference in log-likelihood between models trained on the 9 folds that \emph{do not} contain that data point, so as to avoid overfitting. We then take the mean $\deltaLL$ over the corpus as our final metric.

\newpage
\onecolumn
\section{Additional Results}\label{app:other_figs}

\begin{figure}[!h] \label{fig:slor}
    \centering
    \begin{minipage}{0.55\textwidth}
    \includegraphics[width=\textwidth]{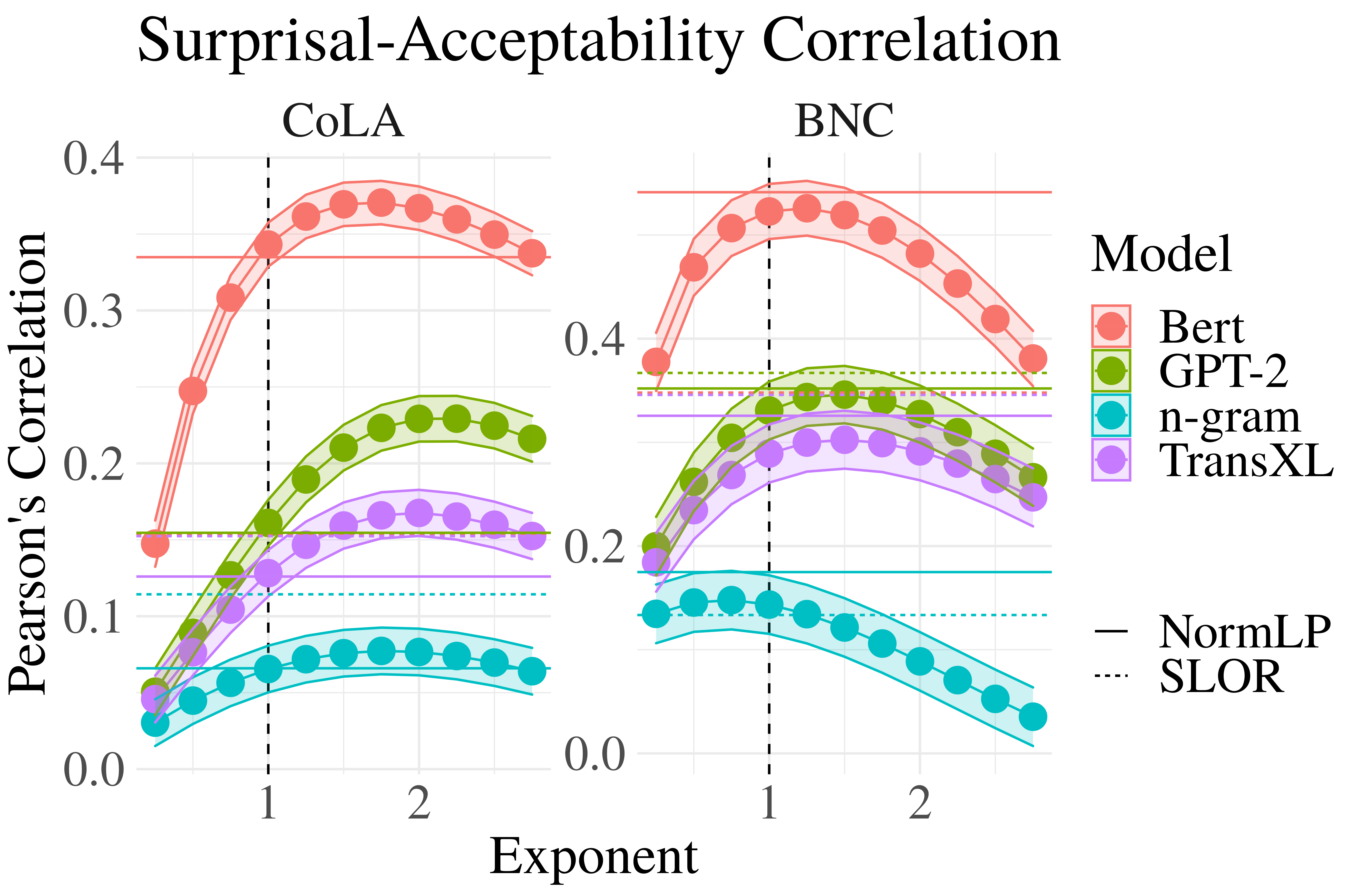}
    \caption{ \cref{fig:cola} with correlations for SLOR and NormLP predictors (from \citet{lau-2017-grammaticality}) }
 \end{minipage}\hfill
    \begin{minipage}{0.4\textwidth}
    \includegraphics[width=0.7\textwidth]{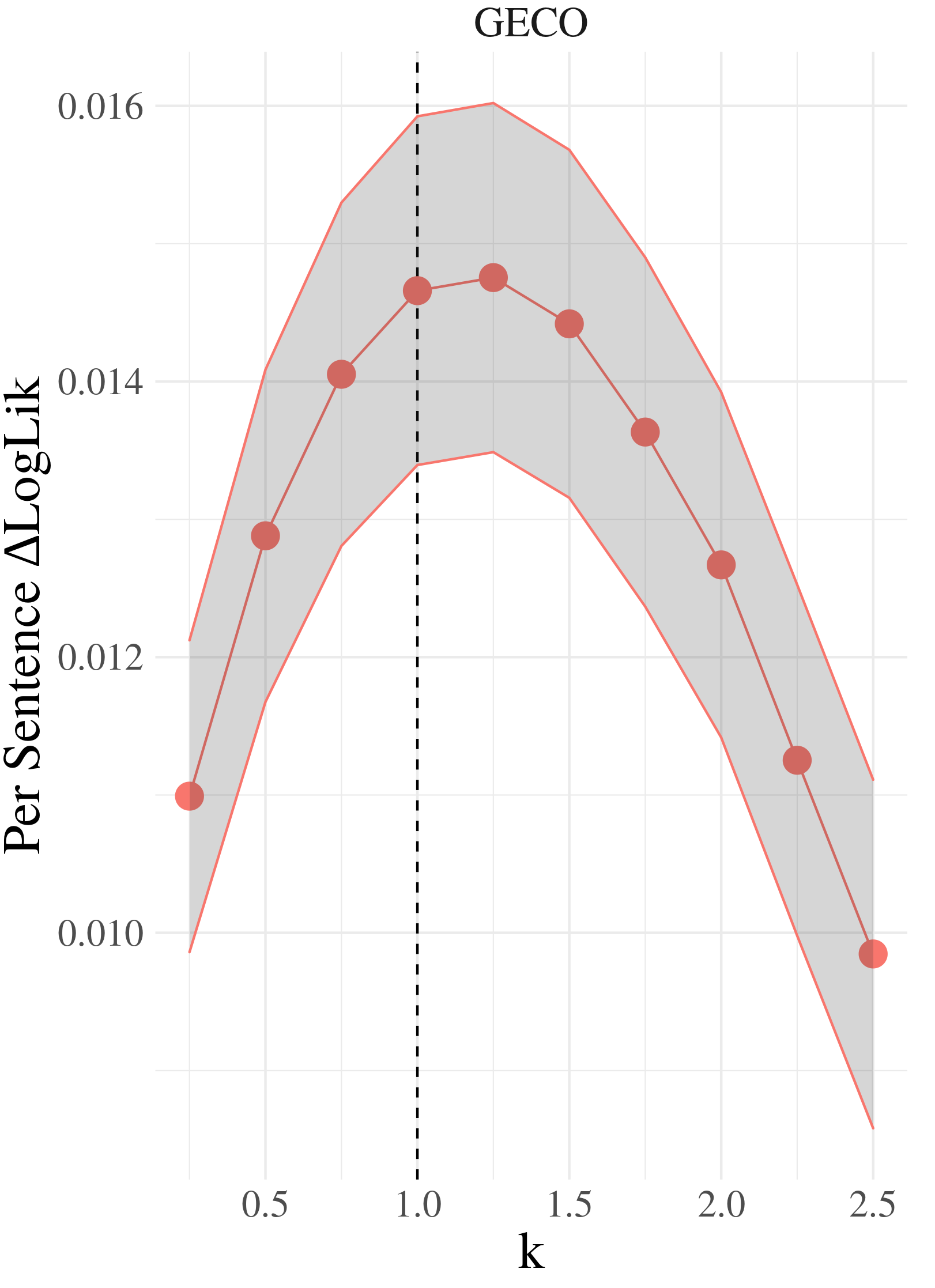}
    \caption{Same graph as in \cref{fig:all} for the Dutch GECO dataset. We use Dutch GPT-2 \cite{devries2020good} for surprisal estimates. }
    \end{minipage}
\end{figure}
\begin{figure}
    \centering
    \begin{minipage}{0.45\textwidth}
        \centering
        \includegraphics[width=0.9\textwidth]{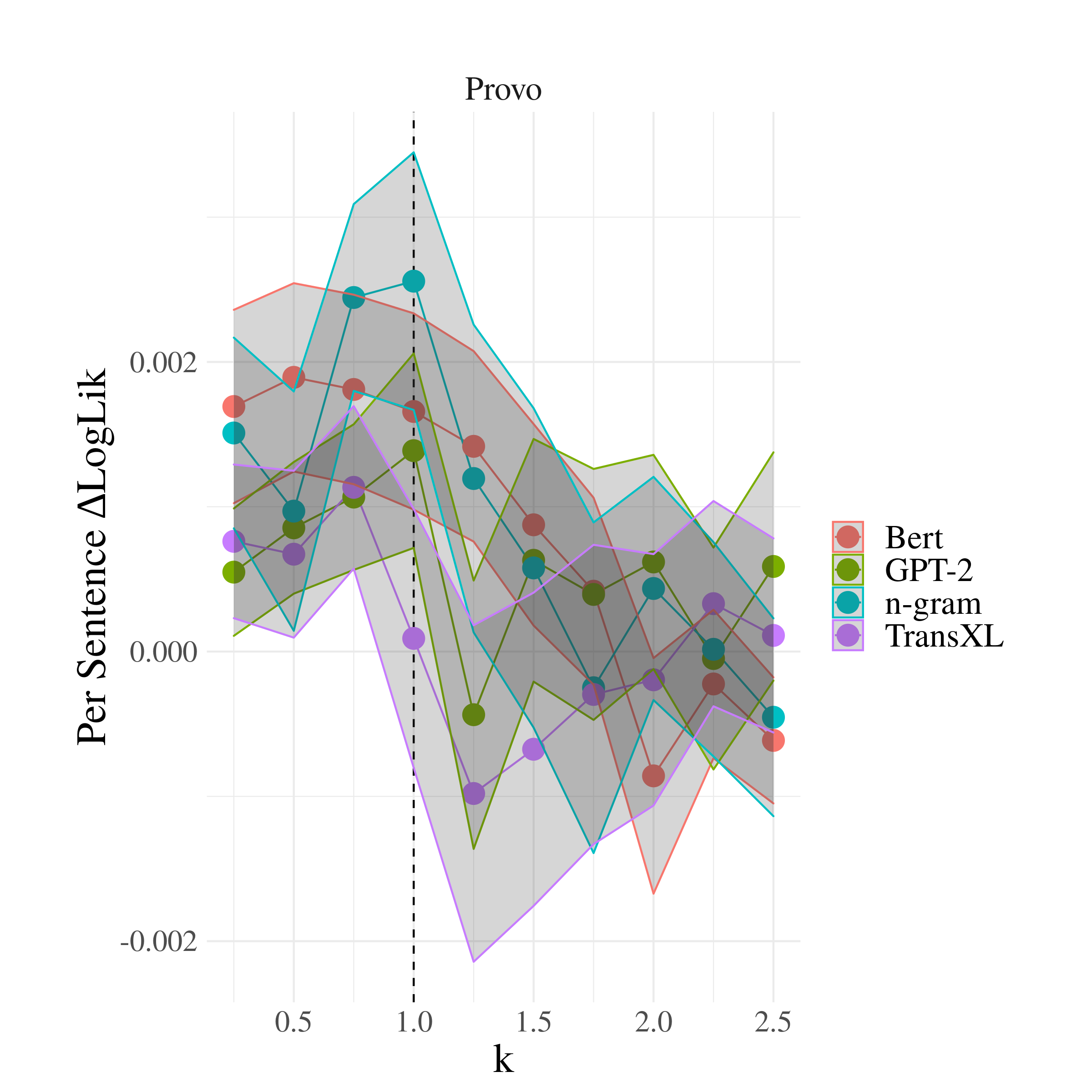} % first figure itself
        \caption{Same graph as in \cref{fig:all} for Provo albeit using (the sum of) first pass times as our reading time metric. }
    \end{minipage}\hfill
    \begin{minipage}{0.45\textwidth}
        \centering
        \includegraphics[width=0.9\textwidth]{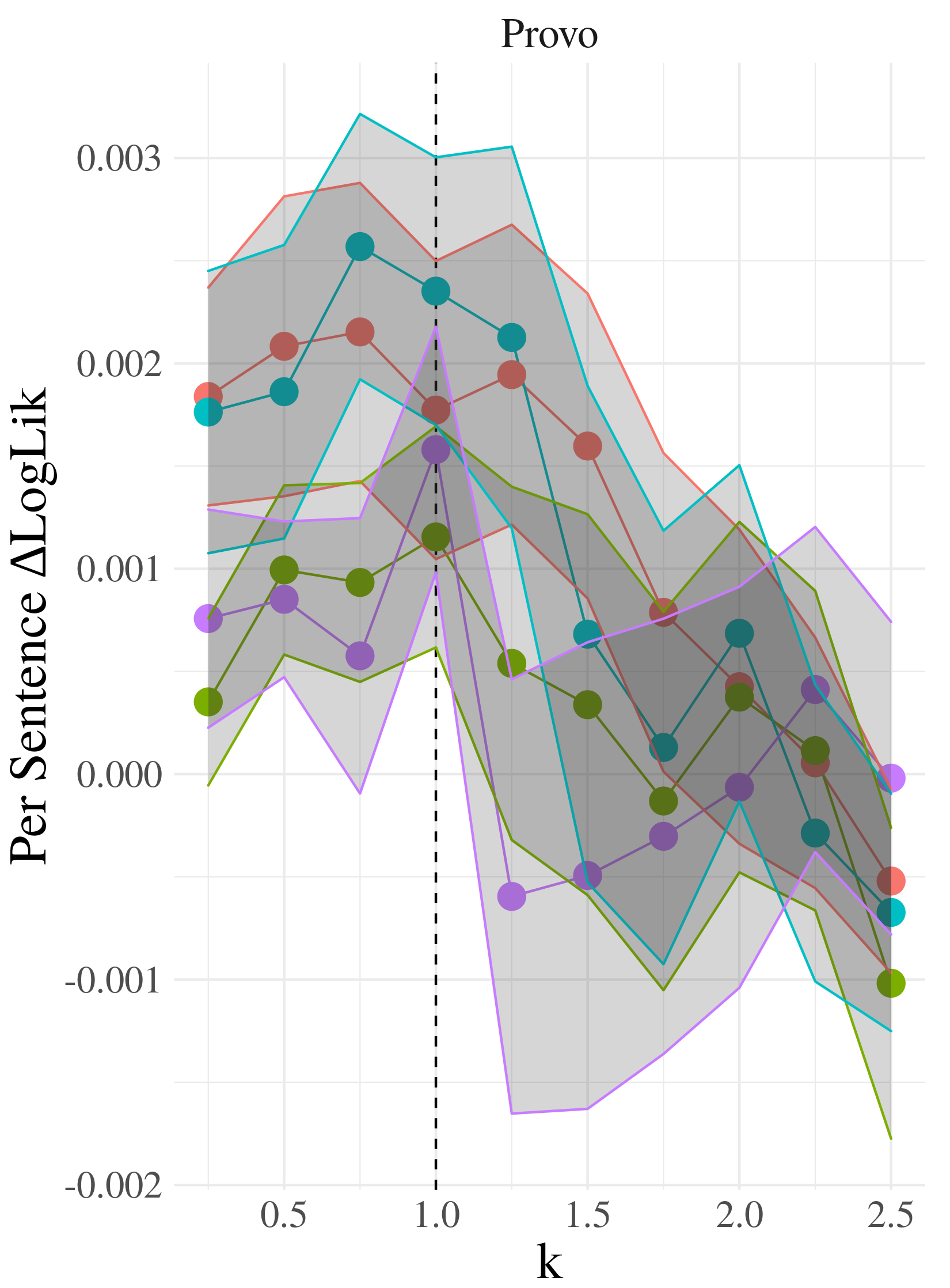} % second figure itself
        \caption{Same graph as in \cref{fig:all} for Provo albeit using (the sum of) first fixation duration times as our reading time metric.}
    \end{minipage}
\end{figure}

\begin{figure*}
    \centering
    \includegraphics[width=\linewidth]{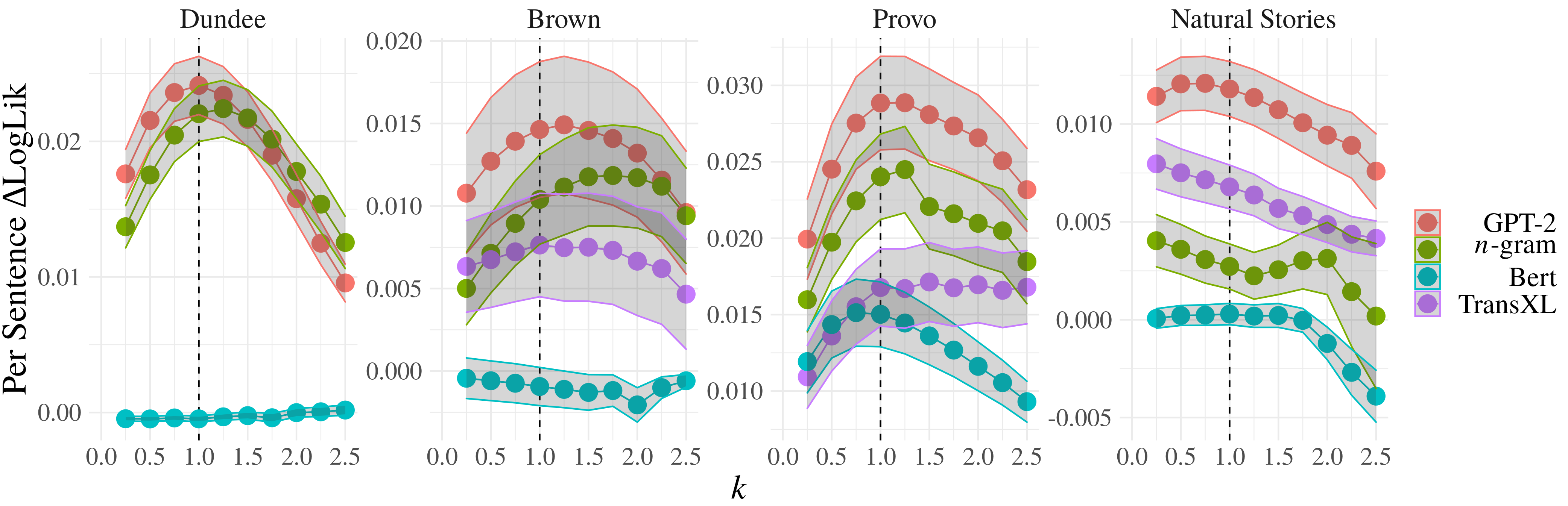}
    \caption{Version of \cref{fig:all} albeit with linear terms for summed unigram log-probability, total character length, and their interaction as predictors.}
\end{figure*}
\begin{figure*} \label{fig:aggregate}
    \centering
    \includegraphics[width=\linewidth]{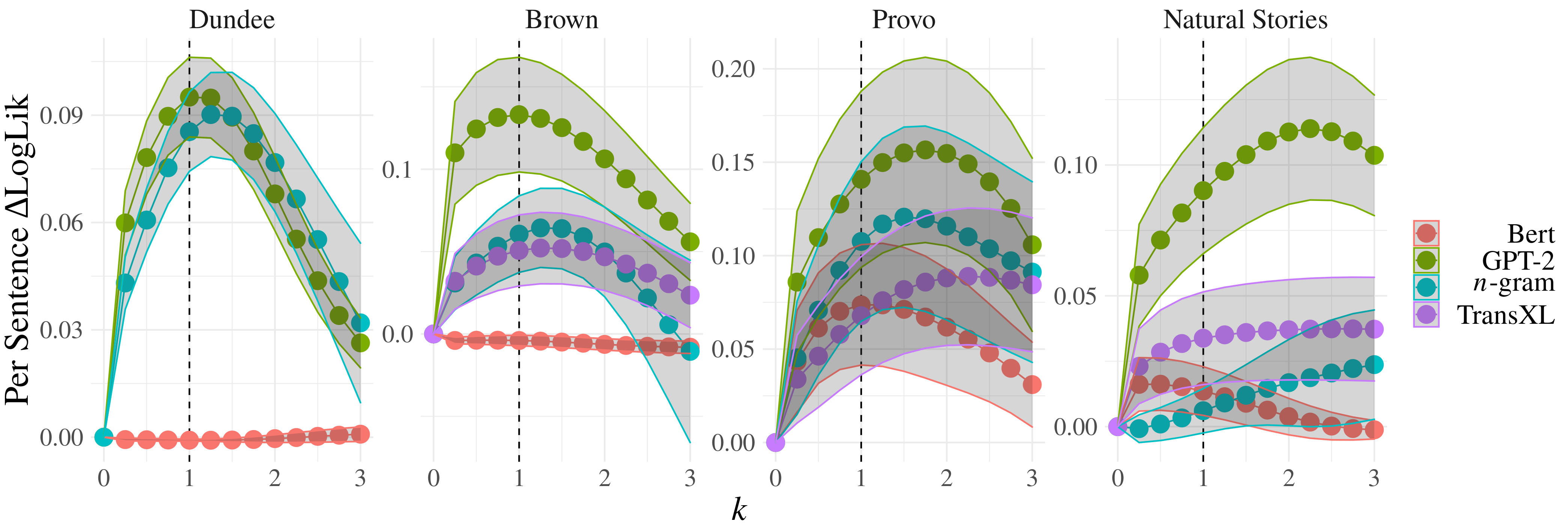}
    \caption{Version of \cref{fig:all} albeit with reading time data aggregated (mean across subjects) per sentence. A simple, linear model is used with the same predictors as \cref{fig:all}}
\end{figure*}

\begin{table*}[!ht]
  \centering
  \small
  \begin{tabular}{ lrrrrrr }
   \toprule
  \multirow{2}{*}{{\bf Predictor}} & \\
      & \multicolumn{1}{c}{Dundee} &\multicolumn{1}{c}{Brown} & \multicolumn{1}{c}{Provo}  &\multicolumn{1}{c}{NS} &  \multicolumn{1}{c}{CoLA} & \multicolumn{1}{c}{\textsc{bnc}} \\
     \hline
$\mathrm{Super}\text{-}\mathrm{Linear}$ ($k=0.25$)	& $	2.08\db{0.2}	$&$	0.88\db{0.36}	$&$	0.97\db{0.21}	$&$	1.05\db{0.11}	$&$	0.9\db{0.03}	$&$	6.11\db{0.13}	$\\
$\mathrm{Super}\text{-}\mathrm{Linear}$ ($k=1$)	& $	2.85\db{0.23}	$&$	1.16\db{0.4}	$&$	1.87\db{0.29}	$&$	1.08\db{0.11}	$&$	5.28\db{0.07}	$&$	13.89\db{0.19}	$\\
$\mathrm{Super}\text{-}\mathrm{Linear}$ ($k=1.25$)	& $	2.83\db{0.23}	$&$	1.16\db{0.41}	$&$	2\db{0.3}	$&$	1.03\db{0.11}	$&$	5.92\db{0.07}	$&$	14.35\db{0.19}	$\\
$\mathrm{Super}\text{-}\mathrm{Linear}$ ($k=1.5$)	& $	2.69\db{0.23}	$&$	1.14\db{0.41}	$&$	1.98\db{0.3}	$&$	0.98\db{0.11}	$&$	6.18\db{0.07}	$&$	14.22\db{0.19}	$\\
$\mathrm{Super}\text{-}\mathrm{Linear}$ ($k=2$)	& $	2.1\db{0.2}	$&$	1.02\db{0.39}	$&$	2.01\db{0.29}	$&$	0.9\db{0.11}	$&$	6.04\db{0.07}	$&$	12.75\db{0.18}	$\\
$\mathrm{Variance}$ (lang)	& $	1.18\db{0.15}	$&$	0.66\db{0.32}	$&$	1.59\db{0.27}	$&$	0.36\db{0.08}	$&$	5.64\db{0.07}	$&$	11.26\db{0.17}	$\\
$\mathrm{Variance}$ (sent)	& $	0.96\db{0.14}	$&$	0.53\db{0.28}	$&$	1.57\db{0.27}	$&$	0.42\db{0.09}	$&$	1.86\db{0.04}	$&$	7.56\db{0.14}	$\\
$\mathrm{Local Variance}$	& $	0.9\db{0.13}	$&$	0.55\db{0.3}	$&$	1.16\db{0.23}	$&$	0.3\db{0.07}	$&$	1.44\db{0.04}	$&$	4.88\db{0.12}	$\\
$\mathrm{Max}$	& $	0.79\db{0.14}	$&$	0.42\db{0.31}	$&$	0.33\db{0.22}	$&$	0.37\db{0.1}	$&$	1.16\db{0.03}	$&$	5\db{0.12}	$\\
$\mathrm{Entropy}$ ($k=0.25$)	& $	1.52\db{0.17}	$&$	0.45\db{0.26}	$&$	1.34\db{0.22}	$&$	0.31\db{0.12}	$&$	0.02\db{0}	$&$	0.03\db{0.01}	$\\
\makecell[l]{$\mathrm{Entropy}$ ($k=1$) {\scriptsize Shannon}}	& $	-0.01\db{0}	$&$	0\db{0.01}	$&$	0.01\db{0}	$&$	0\db{0}	$&$	0\db{0}	$&$	7.9\db{0.14}	$\\
\makecell[l]{$\mathrm{Entropy}$ ($k=2$) {\scriptsize  Renyi}}	& $	-0.01\db{0}	$&$	0\db{0.01}	$&$	0\db{0.01}	$&$	0\db{0}	$&$	0\db{0}	$&$	8.38\db{0.14}	$\\
    \bottomrule
  \end{tabular} 
  \caption{ $\deltaLL$ in $10\text{e-}2$ nats, as in \cref{tab:sentence-level} albeit with different baseline predictors for reading time data and with using BERT for surprisal estimates for acceptability judgments. Along with the predictors specified in \cref{tab:sentence-level}, models for reading times here also contain predictors for unigram log-probability, total character length, and the interaction of the two (reading times). We see largely the same trends as in \cref{tab:sentence-level}.}
  \label{tab:sentence-level-bonus}
\end{table*}

\begin{figure*} 
    \centering
    \includegraphics[width=\linewidth]{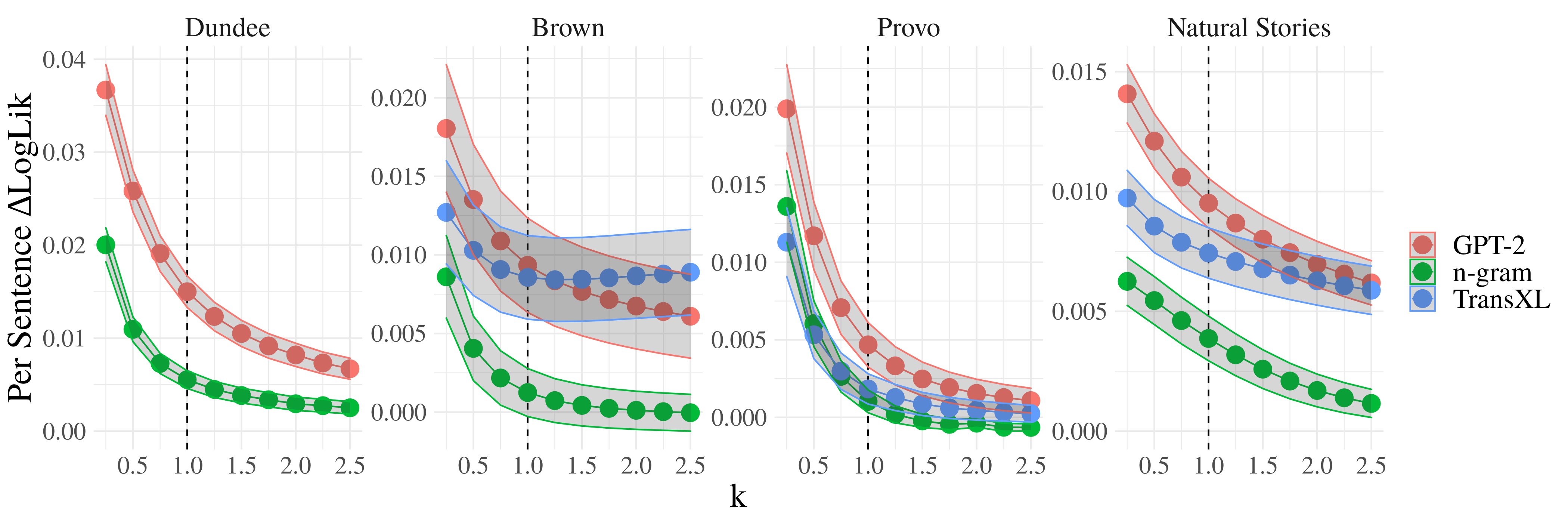}
    \caption{Version of \cref{fig:all} albeit using probabilities instead of surprisal in the summation $\sum_{n=1}^N  s(u_{n})$. Note that the magnitude of $\deltaLL$ is smaller than when using surprisal, indicating the superior predictive power of the latter. This stands in contrast to the experimental findings of \citet{BROTHERS2021104174}.}
    \label{fig:probs}
\end{figure*}

\end{document}